\documentclass[11pt]{article}
\pdfoutput=1
\usepackage[utf8]{inputenc}
\usepackage{amsmath}
\usepackage{amssymb}
\usepackage{xcolor}
\usepackage{color}
\usepackage{graphicx}
\usepackage[algo2e]{algorithm2e}
\usepackage{bbm}
\usepackage[hyphens]{url}

\usepackage[margin=1in, paperwidth=8.5in, paperheight=11in]{geometry}
\usepackage{amsthm}
\newtheorem{theorem}{Theorem}
\newtheorem{corollary}[theorem]{Corollary}
\newtheorem{lemma}[theorem]{Lemma}
\theoremstyle{definition}

\newcommand{\N}{\mathbb{N}}
\newcommand{\R}{\mathbb{R}}
\newcommand{\E}{\mathbb{E}}
\newcommand{\D}{\mathbf{D}}

\DeclareMathOperator*{\argmax}{arg\,max}
\DeclareMathOperator*{\argmin}{arg\,min}
\DeclareMathOperator*{\KL}{KL}

\title{Quantifying the Burden of Exploration\\and the Unfairness of Free Riding}
\author{Christopher Jung \and Sampath Kannan \and Neil Lutz\\\\Department of Computer and Information Science\\University of Pennsylvania, Philadelphia, PA 19104}

\begin{document}
\maketitle

\begin{abstract}
We consider the multi-armed bandit setting with a twist. Rather than having just one decision maker deciding which arm to pull in each round, we have $n$ different decision makers (agents). In the simple stochastic setting, we show that a ``free-riding'' agent observing another ``self-reliant'' agent can achieve just $O(1)$ regret, as opposed to the regret lower bound of $\Omega (\log t)$ when one decision maker is playing in isolation. This result holds whenever the self-reliant agent's strategy satisfies either one of two assumptions: (1) each arm is pulled with high probability at least $\gamma \ln t$ times  for a suitable constant $\gamma$, or (2) the self-reliant agent achieves $o(t)$ realized regret with high probability. Both of these assumptions are satisfied by standard zero-regret algorithms. Under the second assumption, we further show that the free rider only needs to observe the number of times each arm is pulled by the self-reliant agent, and not the rewards realized.

In the linear contextual setting, each arm has a distribution over parameter vectors, each agent has a context vector, and the reward realized when an agent pulls an arm is the inner product of that agent's context vector with a parameter vector sampled from the pulled arm's distribution. We show that the free rider can achieve $O(1)$ regret in this setting whenever the free rider's context is a small (in $L_2$-norm) linear combination of other agents' contexts and all other agents pull each arm $\Omega (\log t)$ times with high probability. Again, this condition on the self-reliant players is satisfied by standard zero-regret algorithms like UCB. We also prove a number of lower bounds.
\end{abstract}

\section{Introduction}

We consider situations where exploitation must be balanced with exploration in order to obtain optimal performance. Typically there is a single decision maker who does this balancing, in order to minimize a quantity called the regret. In this paper we consider settings where there are many agents and ask how a single agent (the \emph{free rider}) can benefit from the exploration of other \emph{self-reliant} agents. For example, competing pharmaceutical companies might be engaged in research for drug discovery. If one of these companies had access to the research findings of its competitors, it might greatly reduce its own exploration cost. Of course, this is an unlikely scenario since  intellectual property is jealously guarded by companies, which points to an important consideration in modeling such scenarios: the amount and type of information that one agent is able to gather about the findings of others.

More realistically, and less consequentially, a recommendation system such as Yelp\texttrademark \  makes user ratings of restaurants publicly available. The assumption underlying such systems is that ``the crowd'' will explore available options so that we end up with accurate average ratings. Many problems of this sort can be modeled using the formalism of multi-armed bandits. Free riding also arises in online advertising. Each advertiser may be modeled as an agent with a context vector describing its likely customers, and it must choose online niches in which to advertise. A free rider can take advantage of competitors' exploration of niches by monitoring impressions and clickthroughs of their ads. In fact, there are a number of paid services (WhatRunsWhere, Adbeat, SpyFu, etc.) that facilitate this behavior.

Multi-armed bandit problems model decision making under uncertainty~\cite{Robbins52,LaiRobbins85,Bubeck12}. Our focus in this paper will be on the \emph{stochastic bandits model} where there is an unknown reward distribution associated with each arm, and the decision maker has to decide which arm to pull in each round. Her goal is to minimize \emph{regret}, the (expected) difference between the reward of the best arm and the total reward she obtains. In the extension to the \emph{linear contextual bandits model}, each arm $i$ has an unknown parameter vector $\theta_i \in \mathbb{R}^d$ for $i = 1,\ldots,k$, where $k$ is the total number of arms. At round $t$, a \emph{context} $x^t \in \mathbb{R}^d$ arrives. The expected reward for pulling arm $i$ in round $t$ is the inner product $\langle \theta_i , x^t \rangle$. 

In the simple stochastic case, there are two types of relevant information: the other agents' actions and the resulting rewards. In the full-information setting, the free rider has access to both types of information. We also consider a partial-information setting where the free rider can only observe the other agents' actions. For linear contextual bandits, the full-information setting also includes the context vectors of the other agents.

In our setting, using Yelp\texttrademark \ as the running example, the $k$ arms correspond to restaurants. Our model differs from standard bandit models in three significant ways. First, there are $n$ decision makers rather than one; in the Yelp\texttrademark \ example, each decision maker corresponds to a diner. Upon visiting a restaurant, a diner samples from a distribution to determine her dining experience. In the stochastic setting, we assume that all diners have identical criteria for assessing their experiences, meaning that identical samples lead to identical rewards. Second, in the linear contextual setting, the contexts in our model are fixed in time and can be regarded as the \emph{types} of the individual decision makers. Each diner's context vector represents the weight she assigns to various features (parameters) of a restaurant, such as innovativeness, decor, noise level, suitability for vegetarians, etc. Third, each arm has a distribution over parameter vectors instead of a fixed parameter vector. When a diner visits a restaurant, her reward is determined by taking the inner product of her context with a parameter vector drawn from the restaurant's distribution, rather than by adding sub-Gaussian noise to the inner product with a fixed parameter vector as in the standard model.

In the standard stochastic or linear contextual bandit setting, a decision-making algorithm is called \emph{zero-regret} if its regret over $t$ rounds is $o(t)$. It is well known that exploration is essential for achieving zero regret~\cite{LaiRobbins85}. One algorithm that achieves the asymptotically optimal regret  bound of $O(\log t)$ over $t$ rounds is the so-called Upper Confidence Bound (UCB) algorithm of Lai and Robbins~\cite{LaiRobbins85}. In addition to maintaining a sample mean for each arm, this algorithm maintains confidence intervals around these means, where the width of the confidence interval for arm $i$ drops roughly as $1/\sqrt{n_i}$ where $n_i$ is the number of times arm $i$ has been pulled so far. The UCB algorithm then selects the arm with the highest upper limit to its confidence interval. There are many other zero-regret strategies, such as Thompson sampling~\cite{thompson1933likelihood} or one where  an initial round-robin exploration phase is followed by an exploitation phase in which the apparently optimal arm is pulled~\cite{GarivierMS16}. 

\paragraph{Our results:}

\begin{itemize}

\item In the stochastic setting a free rider can achieve $O(1)$ regret under either of two reasonable assumptions, both of which are satisfied by standard zero-regret algorithms:
\begin{itemize}
	\item Some self-reliant agent has pulled each arm at least $\gamma \ln t$ times with high probability at all sufficiently large times $t$, where $\gamma$ is a constant derived from our analysis (Theorem~\ref{thm:sampleaugmenting}).
	\item Some self-reliant agent is playing a strategy that with high probability achieves $o(t)$ realized regret. In this case, the free rider can achieve $O(1)$ regret even in the partial-information setting (Theorem~\ref{thm:stoch}). As a corollary, a free rider can achieve $O(1)$ regret whenever a self-reliant agent plays UCB (Corollary~\ref{cor:ucb}).
\end{itemize}

\item For linear contextual bandits, a free rider can again achieve $O(1)$ regret in the full-information setting under an assumption similar to the first assumption above (Theorem \ref{thm:fullinfo}).

\item As a way of relating the two assumptions in the first bullet above, we prove that if a self-reliant agent achieves $O(t^{1 - \epsilon })$ regret, then that agent must pull \emph{each} arm $\Omega(\log t)$ times in expectation (Theorem~\ref{thm:expectedcount}) and with a high probability that depends on $\epsilon$ (Theorem~\ref{thm:highprobcount}).

\item There is  a deterministic lower bound of $\Omega (\log t)$ on the number of times a UCB agent must pull each arm in the stochastic case (Theorem~\ref{thm:ucbcount}).

\item To achieve $o(\log t)$ regret in the contextual setting, the free rider must know both the contexts and the observed rewards of the other agents (Theorems~\ref{thm:needcontexts} and~\ref{thm:needrewards}).
\end{itemize}

\paragraph{Related work:} 
This paper asks how and when an agent may avoid doing their ``fair share'' of exploration. Several recent works have studied how the cost of exploration in multi-armed bandit problems is distributed, from the perspective of algorithmic fairness. Works by Bastani, Bayati, and Khosravi~\cite{bastani2017mostly}; Kannan, Morgenstern, Roth, Waggoner, and Wu~\cite{kannan2018smoothed}; and Raghavan, Slivkins, Vaughan, and Wu~\cite{RaghavanSVW18} show that if the data is sufficiently diverse, e.g., if the contexts are randomly perturbed, then exploration may not be necessary. Celis and Salehi~\cite{celis2017lean} consider a model in both the stochastic and the adversarial setting where each agent in the network plays a certain zero-regret algorithm (UCB in the stochastic setting and EXP3 in the adversarial setting) and study how much information an agent can gather from his neighbors.

There is some discussion in the economics literature of free riding in bandit settings. In the model of Bolton and Harris~\cite{bolton1999strategic}, agents choose what fraction of each time unit to devote to a safe action (exploitation) and to a risky action (exploration), and they show that while the attraction of free riding drives agents to select the safe action always, risky action by a agent may enable everyone to converge to the correct posterior belief faster. Keller, Rady, and Cripps~\cite{keller2005strategic} consider a very similar setting where a risky arm will generates positive payoff after an exponentially distributed random time; they characterize unique symmetric equilibrium as well as various asymmetric equilibria. Klein~\cite{klein2013strategic} gives conditions for complete learning in a two-agent, three-armed bandit setting where there are two negatively-correlated risky arms and a safe arm, with further assumptions about their behavior. It is clear that these models do not support having more than two arms (or three in the case of~\cite{klein2013strategic}) and that their goal is maximizing expected reward, not minimizing regret. Moreover, one arm is explicitly designated as the safe arm and the other(s) as risky, \emph{a priori}.

\section{Preliminaries}
\label{sec:prelims}
\subsection*{Stochastic Model}
There are $k$ \emph{arms}, indexed by $[k]=\{1,\ldots,k\}$ and $n$ \emph{players} or \emph{agents}, indexed by $[n]$. Arm $i$ has a \emph{reward distribution} $D_i$ supported on $[-1,1]$ with mean $\mu_i$, and $\mathbf{D}=(D_1,\ldots,D_k)$ is the \emph{reward distribution profile}, or the \emph{stochastic bandit}. The arm with the highest mean reward is denoted by $i^*=\argmax_{i\in[k]}\mu_i$, and we write $\mu^*$ for $\mu_{i^*}$; we assume that $i^*$ is unique. Letting $\Delta_i = \mu^*-\mu_i$ for each $i \in [k]$, we define define $\Delta=\max_{i\in[k]\setminus\{i^*\}}\Delta_{i}$, the \emph{gap} between optimal and suboptimal arms.

In round $t=1,2,\ldots$, each player $p$ selects an arm $i_p^t\in[k]$ and receives a reward $r_p^t\sim D_{i_p^t}$. We write $H^T= ((i^t_p, r^t_p)_{t\in[T]})_{p \in [n]}$ to denote the \emph{history} of all players' actions and rewards through round $T$.
A \emph{policy} or \emph{strategy} for a player $p$ is a function $f_p$ mapping each history to an arm or to a distribution over the arms; a player $p$ with policy $f_p$ who observes history $H^{T}$ will pull arm $f_p(H^{T})$ in round $T+1$. A \emph{policy profile} is a vector $\mathbf{f}=(f_1,\ldots,f_n)$, where each $f_p$ is a policy for player $p$. Notice that a policy profile and a stochastic bandit together determine a distribution on histories. A policy $f_p$ for player $p$ is \emph{self-reliant} if it depends only on $p$'s own observed actions and rewards. In contrast, a \emph{free-riding} policy may use all players' history.

For any player $p$, arm $i$, and time $t$, the \emph{sample count} is $N_{p,i}^{t}$, the number of times $i$ has been pulled by the player in the first $t$ rounds, and the \emph{sample mean} is $\mu_{p,i}^{t}$, the average of all of player $p$'s samples of arm $i$ through time $t$. 

The \emph{regret} of player $p$ at time $T$ under stochastic bandit $\mathbf{D}$ is $R_p^T(\mathbf{D},\mathbf{f})=\sum_{i \in [k]} \Delta_i \E[N_{p,i}^t]$, where the expectation is according to the distribution on histories determined by $\mathbf{D}$ and $\mathbf{f}$.\footnote{Some sources refer to this quantity as \emph{pseudo-regret}.} When it will not introduce ambiguity, we simply write $R^T_p$ or, in single-player settings, $R^T$. We also consider the \emph{realized regret} under a particular history $H^T$, $\hat{R}_p^T(\mathbf{D},\mathbf{f})=\sum_{i \in [k]} \Delta_i N_{p,i}^t$.

One well-studied self-reliant policy that achieves logarithmic regret in the stochastic setting is called \emph{$\alpha$-UCB}~\cite{LaiRobbins85}, defined by
\[\alpha\text{-UCB}(H^t)=\argmax_{i\in[k]}\mu^t_i + \sqrt{\frac{\alpha \ln(t+1)}{2N^{t}_i}}\]
for all histories $H^t$. The parameter $\alpha$ calibrates the balance between exploration and exploitation. For each arm $i$, a player using this policy maintains an \emph{upper confidence bound} on $\mu_i$, and in each round, she pulls the arm with the highest upper confidence bound. The distance from each arm's sample mean to its upper confidence bound depends its sample count.

\subsection*{Linear Contextual Model}
The linear contextual model generalizes the stochastic model. Now, each arm $i$ has a \emph{feature distribution} $F_i$ supported on the $d$-dimensional closed unit ball, for some $d\in\N$, and $\mathbf{F}=(F_1,\ldots,F_n)$ is the \emph{feature distribution profile} or \emph{contextual bandit}. Each player $p$ has a context $x_p\in\R^d$, and $\mathbf{x}=(x_1,\ldots,x_n)$ is the \emph{context profile}. As before, in each round $t$, each player $p$ selects an arm $i_p^t$, but now the reward is given by sampling a feature vector $\theta_p^t\sim F_{i_p^t}$, and taking its inner product with $x_p$, i.e., $r_p^t=\langle \theta_p^t,x_p\rangle$. $D_{p,i}$ is the distribution of rewards from arm $i$ for player $p$, and the mean of this distribution is
$\mu_{p,i}=\E_{\theta_i\sim F_i}[\langle \theta_i, x_p\rangle].$ The optimal arm for player $p$ is $i^*_p=\argmax_{i\in k}\mu_{p,i}$. Similarly as before, we write $\mu^*_p = \mu_{p, i^*_p}$, $\Delta_{p,i} = \mu_{p,i}^*-\mu_{p,i}$ for each $i \in [k]$, and $\Delta_p=\max_{i\in[k]\setminus\{i^*\}}\Delta_{p,i}$

Histories, policies, policy profiles, self-reliance, and free riding are defined exactly as in the stochastic setting. The regret of player $p$ through round $T\in\N$ under contextual bandit $\mathbf{F}$, context profile $\mathbf{x}$, and policy profile $\mathbf{f}$ is given by $R_p^T(\mathbf{F},\mathbf{x},\mathbf{f})=\sum_{i \in [k]} \Delta_{p,i} \E[N_{p,i}^t]$,
where the expectation is taken according to the distribution determined by $\mathbf{F}$, $\mathbf{x}$, and $\mathbf{f}$. Notice that for a self-reliant player $p$ with context $x_p$, the contextual bandit $\mathbf{F}=(F_1,\ldots,F_n)$ is equivalent to the stochastic bandit $\mathbf{D}=(D_{p,1},\ldots,D_{p,n})$.

\subsection*{Technical Lemmas from Other Sources}

Some of our lower bounding arguments will use the following technical lemmas. $\KL$ denotes the Kullback–Leibler divergence.

\begin{lemma}[Divergence decomposition~\cite{BanditBook}]\label{lem:dd}
	Let $\D = (D_1,\ldots, D_k)$ and $\D'=(D'_1,\ldots,D'_k)$ be stochastic bandits. For any policy $f$ and time $t$,
	\[\KL(H^t(\D,f),H^t(\D',f))=\sum_{i\in[k]}\E[N_i^t(\D,f)]\KL(D_i,D'_i)\,.\]
\end{lemma}

\begin{lemma}[Bretagnolle-Huber inequality~\cite{BreHub79}, presented as in~\cite{BanditBook}]\label{lem:bh}
	Let $P$ and $Q$ be probability measures on the same measurable space, and let $A$ be any event. Then,
	\[P(A) + Q(A^c)\geq\frac12 \exp(-\KL(P,Q))\,,\]
	where $A^c$ is the complement of $A$.
\end{lemma}

\section{Lower Bounds on Sample Counts}
\label{sec:lowerbound}
If a self-reliant player has sampled an arm sufficiently many times, then a free rider with full information can use those samples to find a good estimate of that arm's mean. In this section, we give three lower bounds on the sample counts of each arm.

Theorem~\ref{thm:expectedcount} shows that if a policy guarantees $O(T^{1-\epsilon})$ regret for some positive $\epsilon$, then every arm must be sampled $\Omega(\log T)$ times in expectation. We prove this using the method of Bubeck, Perchet, and Rigollet~\cite{BuPeRi13}, showing via the Bretagnolle-Huber inequality~\cite{BreHub79} that the learner cannot rule out any arm's optimality without sampling that arm $\Omega(\log T)$ times.

\begin{theorem}\label{thm:expectedcount}
	Let $f$ be any self-reliant policy such that $R^T=O(T^{1-\epsilon})$ for all stochastic bandits and some $\epsilon>0$. Then for all stochastic bandits with $\mu^*<1$ and all $i\in[k]$, $f$ satisfies $\E[N_i^T]=\Omega(\log T)$.
\end{theorem}

\begin{proof}
	Let $k\geq 1$ and $\epsilon>0$, and let $f$ be any self-reliant policy satisfying $R^T(\mathbf{D},f)=O(T^{1-\epsilon})$ for all $k$-arm stochastic bandits $\mathbf{D}$. We prove that for all $k$-arm stochastic bandits $\mathbf{D}$ with $\mu^*<1$ and all $i\in[k]$, we have  $\E[N_i^T(\mathbf{D},f)]=\Omega(\ln T)$.
	
	Fix an arm $i$, and let $X=\E[N_i^T(\mathbf{D},f)]$. If $i$ is optimal, then $X\geq T-R^T(\mathbf{D},f)/\Delta$ and the theorem holds. Hence, let $i$ be any suboptimal arm, let $\delta=\min\left\{\Delta,\frac{1-\mu^*}{2}\right\}$, let $p=\frac{1-\mu^*-\delta}{1-\mu_i}$, and let $\D' = (D_1, \dots, D'_i,\dots,D_k)$, where for all $x\in[-1,1]$,
	\[
	D'_i(x) = p \cdot D_i(x) + 1-p\,.
	\]
	Notice that the mean of $D'_i$ is $\mu_i'=\mu^*+\delta$ and that $\KL(D_i,D'_i)\leq\ln(1/p)$. Now,
	\begin{align*}
	\max\{R^T(\mathbf{D},f),R^T(\mathbf{D}',f)\}&\geq \frac{1}{2}(R^T(\mathbf{D},f)+R^T(\mathbf{D}',f))\\
	&\geq\frac{\delta}{2}\sum_{t=1}^T\left(\Pr[i_t(\mathbf{D},f)= i]+\Pr[i_t(\mathbf{D}',f)\neq i]\right)\\
	&\geq \frac{\delta}{4}\sum_{t=1}^T\exp(-{\KL}(H^{t-1}(\mathbf{D},f),H^{t-1}(\mathbf{D}',f)))\\
	&\geq \frac{\delta T}{4}\exp(-{\KL}(H^{T-1}(\mathbf{D},f),H^{T-1}(\mathbf{D}',f)))\,,
	\end{align*}
	where $\KL$ denotes Kullback-Leibler divergence and the second-to-last line follows from Lemma~\ref{lem:bh}, the Bretagnolle-Huber inequality~\cite{BreHub79}. By Lemma~\ref{lem:dd}~\cite{BanditBook},
	\begin{align*}
	{\KL}(H^{t-1}(\mathbf{D},f),H^{t-1}(\mathbf{D}',f)
	&=\KL(D_i,D'_i)\cdot\E[N^T_i(\D,f)]\\
	&\leq X\ln(1/p)\,.
	\end{align*}
	Thus, we have
	\[\max\{R^T(\mathbf{D},f),R^T(\mathbf{D}',f)\}\geq\frac{\delta T}{4}\exp\left(-X\ln(1/p)\right)\,.\]
	It follows that $\exp(-X\ln(1/p))=O(T^{-\epsilon})$ and therefore that $X=\Omega(\ln T)$.
\end{proof}

In addition to a bound on the expected sample count, we sometimes need stronger guarantees on the tail of the sample count distribution. In Theorem~\ref{thm:highprobcount}, we use a coupling argument to show that if a policy has regret $O(T^{1-\beta})$ for relatively large $\beta$, then the probability that any arm that is sampled too few times is small.

\begin{theorem}\label{thm:highprobcount}
Let $f$ be any self-reliant policy such that $R^T=O(T^{1-\beta})$ for all stochastic bandits and some $\beta>0$. Then for all stochastic bandits with $\mu^*<1$, all $i \in [k]$, and all $\gamma>0$, $f$ satisfies
\[
\Pr\left(N^T_i \leq \gamma \ln T\right) = O\big(T^{\gamma c_i-\beta}\big)\,,
\]
where $ c_i = \ln\left(\frac{1 - \mu^*}{2(1-\mu_i)}\right)$.
\end{theorem}
\begin{proof}
Let $\alpha,\beta,\gamma,t_0>0$, and let $f$ be a self-reliant policy such that for all stochastic bandits $\D$ and all $T>t_0$, $R^T(\D,f)\leq \alpha T^{1-\beta}$. Let $t_1\geq t_0$ satisfy $\gamma\ln t_1<t_1/2$.
Assume for contradiction that there is some stochastic bandit $\D$ with $\mu^*<1$, some arm $i\in[k]$, and some $T>t_1$ such that
\begin{align}\label{eqn:lowerbound_contradiction}
\Pr\left(N^T_i(\D,f) \leq \gamma \ln T\right) > \frac{C\alpha}{T^{\beta+\gamma \ln p_i}}\,,
\end{align}
where $C=\frac{2}{\min\{\Delta,(1-\mu^*)/2\}}$ and $p_i = \frac{1 - \mu^*}{2(1-\mu_i)}$.

Observe that if $i$ is the optimal arm in $\D$, then since $\gamma\ln p_i<0$, we have
\begin{align*}
R^T(\D,f)&> \Delta\cdot(T-\gamma\ln T)\cdot \frac{C\alpha}{T^{\beta+\gamma\ln p_i}}\\
&\geq \Delta\cdot(T-\gamma\ln T)\cdot C\alpha T^{-\beta}\\
&\geq \alpha T^{1-\beta}\,,
\end{align*}
contradicting the assumption that $R^T(\D,f)\leq \alpha T^{1-\beta}$. Hence, we assume that $i$ is suboptimal.

We now construct a stochastic bandit $\D'$ in which $i$ is optimal. Let $\D' = (D_1, \dots, D'_i,\dots,D_k)$, where
\[
D'_i(x) = p_i \cdot D_i(x) + 1-p_i
\]
for all $x\in[-1,1]$. Notice that the mean of $D'_i$ is $\mu_i'=\frac{1+\mu^*}{2}$, and that the gap between optimal and suboptimal arms in $\D'$ is $\Delta'=\mu_i'-\mu^*=\frac{1-\mu^*}{2}$.

We now use a coupling argument to bound $\Pr\left(N_{i}^T(\D',f) \leq \gamma \ln T\right)$. Observe that to sample from $D'_i$, one can sample a reward $x \sim D_i$, keep $x$ with probability $p_i$, and otherwise output $1$. Thus, for any history $h$ in which $i$ is pulled exactly $s$ times, 
\[
\Pr\left(H^T(\D',f)=h\right) \ge p_i^s \cdot \Pr\left(H^T(\D,f)=h\right)\,.
\]
By summing over all such histories, we have
\[\Pr\left(N^T_i(\D',f)=s\right)\geq p_i^s\cdot \Pr\left(N^T_i(\D,f)=s\right)\,,\]
and therefore
\begin{align*}
	\Pr\left(N^T_i(\D',f)\leq\gamma\ln T\right)&=\sum_{s=0}^{\lfloor\gamma \ln T\rfloor}\Pr\left(N^T_i(\D',f)=s\right)\\
	&\geq \sum_{s=0}^{\lfloor\gamma \ln T\rfloor}p_i^s\cdot \Pr\left(N^T_i(\D,f)=s\right)\\
	&\geq p_i^{\gamma\ln T}\sum_{s=0}^{\lfloor\gamma \ln T\rfloor}\Pr\left(N^T_i(\D,f)=s\right)\\
	&=T^{\gamma\ln p_i}\cdot\Pr\left(N^T_i(\D,f) \leq \gamma \ln T \right)\,.
\end{align*}
Combining this bound with inequality (\ref{eqn:lowerbound_contradiction}) yields
\begin{align*}
\Pr\left(N_{i}^T(\D',f) \leq \gamma \ln(T)\right) &> T^{\gamma \ln p_i} \cdot \frac{C\alpha}{T^{\beta+\gamma \ln p_i}}\\
&=C\alpha T^{-\beta}\,.
\end{align*}
Thus,
\begin{align*}
R^T(\D',f) &> C\alpha T^{-\beta} \cdot (T-\gamma\ln T) \cdot \frac{1-\mu^*}{2}\\
&\geq\alpha T^{1-\beta}\,,
\end{align*}
which again contradicts the assumed regret bound on $f$.

\end{proof}

Finally, we use a delicate inductive argument to prove the following deterministic guarantee on the sample count for each arm when the arms are pulled according to the $\alpha$-UCB policy.
\
\begin{theorem}\label{thm:ucbcount}
	Let $\alpha>0$ and $\eta>2$. There exists a constant $t_0$ such that for all stochastic bandits, all $t\geq t_0$, and all $i\in[k]$, an agent playing the $\alpha$-UCB policy must satisfy $N_i^{t-1}\geq \alpha\ln t/(2\eta^2 k^2)$.
\end{theorem}
\begin{proof}
	For every $j\in[k]$ and $t\in\N$, define the set \[U^t_j=\left\{i\in[k]:N_i^{t-1}\geq\frac{\alpha\ln t}{2\eta^2 j^2}\right\}\,.\]
	We claim that for all $j\in[k]$ there is a constant $t_j$ such that for all $t\geq t_j$, $|U^t_j|\geq j$.
	
	We will prove this claim by induction on $j$. For any time $t$, there is clearly some arm $i$ with $N_i^{t-1}\geq\frac{t-1}{k}$, and we can choose $t_1$ such that $\frac{t-1}{k}\geq \frac{\alpha\ln t}{2\eta^2 k^2}$ whenever $t\geq t_1$, so the claim holds for $j=1$.
	
	Now fix $j>1$, and assume that the claim holds for $j-1$. Define a function $g_j:\N\to\R$ by
	\[g_j(t)=t-(k-j+1)\frac{\alpha\ln t}{2\eta^2 j^2}\,.\]
	We choose $t_j$ sufficiently large such that for all $t\geq t_j$ we have $g_j(t)> t_{j-1}$ and
	\begin{equation}\label{eq:gjcond}
	\frac{\ln(g_j(t)-1)}{\ln t}> \left(1-\frac{1-2/\eta}{j}\right)^2\,.
	\end{equation}
	
	Assume for contradiction that there is some time $t\geq t_j$ such that  $|U^t_j|< j$. Since $U^t_{j-1}\subseteq U^t_j$,
	the inductive hypothesis then implies that $U^t_j=U^t_{j-1}$. Thus, $|U_j^t|=j-1$, and there are exactly $k-j+1$ arms outside of $U^t_j$. Each one of those arms has been pulled at most $\frac{\alpha\ln t}{2\eta^2 j^2}$ times by round $t-1$, so by the pigeonhole principle there is some $s\in\left[g_j(t)-1,t-1\right]$ such that an arm from $U_j^t$ is pulled in round ${s}$.
	
	Furthermore, inequality (\ref{eq:gjcond}) implies
	\[\frac{\alpha\ln s}{2\eta^2 (j-1)^2}> \frac{\alpha\ln t}{2\eta^2 j^2}\,,\]
	which guarantees that $U^s_{j-1}\subseteq U_{j}^t$. Since $s\geq g_j(t_j)-1\geq t_{j-1}$, the inductive hypothesis tells us that $|U_{j-1}^s|\geq {j-1}$, so we have $U_{j-1}^s=U_j^t$, meaning that the arm pulled in round $s$ is also in $U_{j-1}^s$.
	
	Now, $N_i^{s-1}\geq \frac{\alpha\ln s}{2\eta^2 (j-1)^2}$ for all $i\in U_{j-1}^s$, so the upper confidence bound of the arm pulled at time $s$ is at most
	\[1+\sqrt{\frac{\alpha\ln s}{2\alpha\ln s/(2\eta^2 (j-1)^2)}}=1+\eta\cdot(j-1)\,.\]
	The upper confidence bound at time $s$ of any arm in $[k]\setminus U_j^t$ is at least
	\[-1+\sqrt{\frac{\alpha\ln s}{2\alpha\ln t/(2\eta^2 j^2)}}=-1+\eta j\sqrt{\frac{\ln s}{\ln t}}\,.\]
	But since $t\geq t_{j}$ and $s\geq g_j(t_{j})$, inequality~(\ref{eq:gjcond}) implies
	\[-1+\eta j\sqrt{\frac{\ln s}{\ln t}}>1+\eta\cdot(j-1)\,.\]
	This means that all arms outside of $U_j^t$ have higher upper confidence bounds at time $s$ than the arms in $U_j^t$, contradicting the choice to pull an arm in $U_j^t$ at time $s$.
	
	By induction, we conclude that the claim holds for all $j\in[k]$, and in particular that the theorem holds with $t_0=t_k$.
\end{proof}

\section{Free Riding with Stochastic Bandits}
\label{sec:stochastic}

\subsection*{Full-Information Case for Stochastic Bandits}
Here we describe how a free rider can take advantage of the samples collected by another player $p$ who pulls every arm sufficiently many times with high probability. This free-riding policy, which we call $\textsc{SampleAugmentingMeanGreedy}$, divides time into \emph{epochs} of doubling length. In the $j$\textsuperscript{th} epoch, the free rider checks whether a given player $p$ has observed at least $\gamma j$ samples of each arm $i\in[k]$, where $\gamma$ is an appropriate constant. If all sample counts are sufficient, then the free rider uses $p$'s observed rewards to estimate the mean of each arm, committing to the arm with the maximum estimated mean for the remainder of the epoch. Otherwise, the free rider pulls any under-sampled arms, augmenting all sample counts up to at least $\gamma j$ before proceeding. 
Doing this allows the free rider to circumvent the logarithmic lower bound on regret and achieve $O(1)$ regret.
\begin{theorem}\label{thm:sampleaugmenting}
	Fix a stochastic bandit, and suppose some player $p$ plays a self-reliant policy that satisfies $\Pr(N_i^{t-1}<\gamma\ln t)=O((\log t)^{-w})$ for some $\gamma>2\ln(2)/\Delta^2$, some $w>2$, and all $i\in[k]$. Then a free rider can achieve $O(1)$ regret.
\end{theorem}

Below, $\hat{\mu}^s_{p,i}$ denotes the average of the first $s$ samples of $D_i$ observed by player $p$.
\begin{algorithm2e}
\SetAlgoLined
\caption{$\textsc{SampleAugmentingMeanGreedy}_{p,\gamma}$}
\For{$j \in\mathbb{N}$}{
	$t=2^j$\\
	\For{$i\in[k]$}{
		$N=N_{p,i}^{2^j-1}$\\
		\If{$N \geq \gamma j$}{
			\tcp{$p$ sampled arm $i$ enough prior to this epoch}
			$\nu^j_i=\hat{\mu}_{p,i}^{s_j}$}
		\Else{\tcp{the free rider samples arm $i$ up to $\lceil\gamma j\rceil-N$ times}
			$\nu^j_i=\hat{\mu}_{p,i}^{N}\cdot N/(\lceil \gamma j\rceil)$\\
			\While{$N<s_j$ and $t < 2^{j+1}-1$}{
				$i_1^t=i$\\
				$\nu^j_i=\nu^j_i+r^t_1/s$\\
				$N=N+1$\\
				$t = t + 1$\\
			}
		}
	}
	\While{$t<2^{j+1}-1$}{
		$i_1^t=\argmax_i \nu^j_i$\\
		$t=t+1$}
}
\end{algorithm2e}

\begin{proof}
	Let $\D$ be a stochastic bandit and let $f$ be a self-reliant policy. Assume there are constants $\gamma>2\ln(2)/\Delta^2$, $w>2$, $c>0$, and $t_0$ such that \[\Pr\left(N_{i}^{t-1}(\D,f)<\gamma\ln t\right)< c\cdot(\log t)^{-w}\]
	holds for all $i\in[k]$ and all $t\geq t_0$.

	Let $\mathbf{f}=(f_1,\ldots,f_n)$ be a policy profile with $f_1=\textsc{SampleAugmentingMeanGreedy}_{p,\gamma}$ and $f_p=f$, for some player $p\in\{2,\ldots,n\}$. We prove that  $R_1^T(\D,\mathbf{f})=O(1)$.

	Let $t_0'= 2^{\lceil\log t_0\rceil}$. For all $j\in\N$, define $s_j=\gamma j$ and $\tilde{i}^*_j = \argmax_i \nu^j_i$. Then, we have
	\begin{align*}
		R_1^T(\D,\mathbf{f})&=\sum_{i\neq i^*} \E\left[N_{1,i}^T(\D,\mathbf{f})\right]\Delta_i\\
		&\leq 2t_0'+\sum_{i\neq i^*}\Delta_i\left(\sum_{t=t_0'}^T\Pr[i_1^t=i]\right)\,.
	\end{align*}
	Now, for each $i\in[k]$,
	\begin{align*}
		\sum_{t=t_0'}^T \Pr(i_1^t=i)&<\sum_{j=\lceil\log t_0\rceil}^{\lceil\log T\rceil} \left(\Pr\left(N_{i}^{2^j-1}(\D,f)<s_j\right)\cdot s_j+\Pr\left(\tilde{i}^*_j=i\right)\cdot 2^{j}\right)\\
		&<c\gamma \sum_{j=\lceil\log t_0\rceil}^{\lceil\log T\rceil} j^{1-w}+\sum_{j=\lceil\log t_0\rceil}^{\lceil\log T\rceil}\Pr\left(\tilde{i}^*_j=i\right)\cdot 2^{j}\,.
	\end{align*}
	Since $1-w<-1$, the first sum converges to a constant. The second is
	\begin{align*}
		\sum_{j=\lceil\log t_0\rceil}^{\lceil\log T\rceil}\Pr\left(\tilde{i}^*_j=i\right)\cdot 2^{j}&\leq \sum_{j=\lceil\log t_0\rceil}^{\lceil\log T\rceil}\left( \Pr\left(\nu_i^j-\mu_i>\Delta_i/2\right)+\Pr\left(\mu^*-\nu_{i^*}^j>\Delta_i/2\right)\right)\cdot 2^{j}\\
		&\leq 2\sum_{j=0}^{\infty} \exp\left(-2\left( \frac{\Delta_i}{2}\right)^2\gamma j\right)\cdot 2^j\\
		&\leq 2\sum_{j=0}^{\infty}\exp((\ln 2-\Delta^2\gamma/2)j)\,.
	\end{align*}
	Since $\Delta^2\gamma/2>\ln 2$, this sum also converges to a constant, so $R_1^T(\D,\mathbf{f})=O(1)$.
\end{proof}

\subsection*{Partial-Information Case for Stochastic Bandits}
We now show that a free rider can achieve constant regret by observing a player who plays any policy that is unlikely to pull suboptimal arms too often. This class of policies includes UCB. We consider a specific, natural free-riding policy $\textsc{CountGreedy}_p$, defined by
\[\textsc{CountGreedy}_p(H^{t})=\argmax_{i\in[k]} N_{p,i}^{t}\,,\]
which always pulls whichever arm $i$ has been pulled most frequently by player $p$. Notice that this policy does not require the free rider to observe player $p$'s rewards.

One might suspect that it would be sufficient for player $p$'s policy to achieve $R_p^T = o(T)$ in order for the free rider to achieve constant regret under $\textsc{CountGreedy}$, but this turns out not to be the case. It is possible for $p$ to achieve logarithmic regret despite frequently pulling suboptimal arms with non-trivial probability, preventing $\textsc{CountGreedy}$ from achieving constant regret.

For instance, consider the self-reliant policy that, for $j=0, 1, 2\dots$, dictates the following behavior in epochs of tripling length. With probability $1/3^j$, player $p$ ``gives up'' on rounds $3^j$ to $3^{j+1}-1$, choosing arm $i^t_p=\argmin_i\mu^{3^j-1}_i$. Otherwise, with probability $1-1/3^j$, player $p$ plays $\alpha$-UCB during those rounds. Under this policy, $p$'s regret grows at most logarithmically.
Notice that whenever player $p$ gives up, $i_p^{3^j}$ will become her most frequently pulled arm by round $2\cdot 3^j$, so the $\textsc{CountGreedy}_p$-playing free rider will pull this arm at least $3^j$ times before round $3^{j+1}$.
It is routine to show that $i_p^{3^j}$ is suboptimal with probability $1-O(1/3^j)$, so the free rider's regret through round $T$ is at least $R_1^T\geq\sum_{j=0}^{\lfloor \log_3 T\rfloor}(1-O(1/3^j))\cdot\Delta\cdot 3^j
=\Omega(T)$.

Notice that by Theorem~\ref{thm:ucbcount}, the above policy satisfies the conditions of Theorem~\ref{thm:sampleaugmenting} when $\alpha$ is sufficiently large, and therefore a free rider playing $\textsc{SampleAugmentingMeanGreedy}_p$ would achieve constant regret in this situation. Intuitively, this is because the policy of sometimes giving up on entire epochs is not ``rational,'' and $\textsc{SampleAugmentingMeanGreedy}$, unlike $\textsc{CountGreedy}$, does not make any implicit assumption of rationality for the self-reliant player.

Since logarithmic regret for player $p$ is not a strong enough assumption, we instead show that if the realized regret $\hat{R}_p^T$ is sublinear with sufficiently high probability, then the free rider achieves constant regret by playing $\textsc{CountGreedy}_p$.

\begin{theorem}\label{thm:stoch}
	Fix a stochastic bandit and assume there is some player $p$ such that for all $\epsilon>0$ there exists a $w>1$ satisfying $\Pr(\hat{R}^T_p\geq \epsilon T)=O(T^{-w})$. Then a free rider playing $\textsc{CountGreedy}_p$ achieves $O(1)$ regret.
\end{theorem}

\begin{proof} 
	Let $\mathbf{D}$ be a stochastic bandit, let $p\in\{2,\ldots,n\}$ be a player, and let $\mathbf{f}=(f_1,\ldots,f_n)$ be a policy profile with $f_1=\textsc{CountGreedy}_p$. Assume that for all $\epsilon>0$ there is some $w>1$ satisfying $\Pr(\hat{R}_p^T(\mathbf{D},\mathbf{f}) \geq \epsilon T)=O(T^{-w})$. We prove that $R^T_1(\mathbf{D},\mathbf{f})=O(1)$.

	The free rider pulls a suboptimal arm at each time $t$ if and only if $\textsc{CountGreedy}_p(H^{t-1})\neq i^*$, which implies that $N^{t-1}_{p,i^*}\leq\frac{t-1}{2}$ and therefore
	$\hat{R}_p^T(\mathbf{D},\mathbf{f}) \geq\Delta \frac{t-1}{2}$.
	Hence, we can bound the free rider's regret by
	\[R^T_1(\mathbf{D},\mathbf{f}) \leq 2\sum_{t=0}^{T-1} \Pr(\hat{R}^t_p(\mathbf{D},\mathbf{f}) \geq \Delta t/2)\,.\]
	If $w>1$ satisfies $\Pr(\hat{R}^t_p(\mathbf{D},\mathbf{f}) \geq \Delta t/2) = O(t^{-w})$, then we have
	\[R^T_1(\mathbf{D},\mathbf{f}) \le 2\sum_{t=0}^{T-1} O(t^{-w})=O(1)\,.\]
\end{proof}

Audibert, Munos, and Szepesv\'{a}ri~\cite{audibert2009exploration} (Theorem 8) showed that $\alpha$-UCB satisfies the probability bound of Theorem~\ref{thm:stoch} in the single-player setting whenever $\alpha>1$. Since $\alpha$-UCB is a self-reliant policy, this immediately yields the following corollary.

\begin{corollary}\label{cor:ucb}
	If some player $p$'s policy is $\alpha$-UCB for any $\alpha > 1$, then a free rider playing $\textsc{CountGreedy}_p$ achieves $O(1)$ regret.
\end{corollary}

\section{Free Riding with Contextual Bandits}
\label{sec:contextual}
Theorems~\ref{thm:sampleaugmenting} and~\ref{thm:stoch} show that free riding is easy in the stochastic case, in which the reward distribution is identical for all players, but the task is more nuanced when players may have diverse contexts. In the linear contextual setting, different players may have different optimal arms, so a simple free-riding strategy like \textsc{CountGreedy} may fail, even when there are strong regret guarantees for the other players. In fact, as we show in Theorems~\ref{thm:needcontexts} and~\ref{thm:needrewards}, successful free riding in this setting requires knowledge of both the contexts and the rewards of other players.

\subsection*{Full-Information Cases for Contextual Bandits}

We now consider the full-information setting where the free rider knows other players' contexts, actions, and rewards. We show that if the free rider's context is a linear combination of the other players' contexts --- and if other players pull all arms sufficiently many times --- then the free rider can aggregate other players' observations to estimate the means of its own reward distribution profile. In the event that some arm has not been sampled enough by some player, the free rider temporarily acts self-reliantly and chooses arms according to UCB. Under the above assumptions, this free-riding policy, $\textsc{UCBMeanGreedy}$, achieves $O(1)$ regret. Formally, for every history $H^t$,
\[\textsc{UCBMeanGreedy}_{\gamma,\mathbf{c}}(H^t)=\begin{cases}
\argmax_i\sum_{p=2}^n c_p\hat{\mu}_{p,i}^{\lceil\gamma j\rceil}&\text{if }S_j\\
2\text{-UCB}\left(\big(i^{2^{j}}_1,r^{2^{j}}_1\big),\ldots,\big(i^{t}_1,r^{t}_1\big)\right)&\text{otherwise}\,,\end{cases}\]
where $j=\lfloor\log t\rfloor$, $S_j$ is the event that $N_{p,i}^{2^{j}-1}\geq\gamma j$ for all $p\in\{2,\ldots,n\}$ and all $i\in[k]$, and $\hat{\mu}^s_{p,i}$ denotes the average of the first $s$ observed samples of $D_{p,i}$. Notice that when applying UCB this policy treats the bandit as stochastic and ``starts from scratch'' in each epoch, only considering the free rider's own actions and observed rewards from the current epoch.

\begin{theorem}\label{thm:fullinfo}
	Let $\mathbf{x}$ be a context profile and $\mathbf{c}\in\R^{n-1}$ be a vector such that $\sum_{p=2}^n c_p x_p=x_1$,
	Fix a contextual bandit, let $\Delta$ be the gap for player 1, and suppose that $\epsilon>0$ and $\gamma>\frac{8\langle\mathbf{c},\mathbf{c}\rangle\ln 2}{\Delta^2}$ satisfy \[\Pr(N_{p,i}^{t-1}<\gamma \log t) = O((\log t)^{-2-\epsilon})\] for every player $p\in\{2,\ldots,n\}$ and arm $i\in[k]$. Then a free rider playing $\textsc{UCBMeanGreedy}_{\gamma,\mathbf{c}}$ achieves $O(1)$ regret.
\end{theorem}

\begin{proof}
Fix $j\in\N$, let \emph{epoch} $j$ be rounds $2^j$ through $2^{j+1}-1$, and let $s_j=\lceil\gamma j\rceil$. For each $i\in[k]$ let $\tilde{\mu}^j_i=\sum_{p=2}^n c_p\hat{\mu}_{p,i}^{s_j}$, and let $\tilde{i}^j=\argmax_{i} \tilde{\mu}^j_i$. We analyze the free rider's regret incurred during epoch $j$ in each of the following cases: (1) $S_j$ and $\tilde{i}^j = i_1^*$, (2) $S_j$ and $\tilde{i}^j \neq i_1^*$, and (3) $\neg S_j$.

Notice that $\tilde{i}^j$ is defined in the first and second cases. In case (1), the free rider incurs no regret during epoch $j$ because it pulls only the optimal arm $i_1^*$ during that epoch. We now analyze the regret incurred from the other two cases.

For case (2),

	\[\Pr(\tilde{i}^j \neq i^*_1)\le \sum_{i\in[k]}\Pr\left(\left|\tilde{\mu}^{j}_i-\mu_{1,i}\right|\geq \frac{\Delta}{2}\right)\,.\]
	For fixed $i\in[k]$, letting $\hat{r}^s_{p,i}$ denote the value of the $s$\textsuperscript{th} observed sample of $D_{p,i}$, we have
	\begin{align*}
	\left|\tilde{\mu}^{j}_i-\mu_{1,i}\right|&=\left|\sum_{p=2}^n \frac{c_p}{s_j}\sum_{s=1}^{s_j}\hat{r}^s_{p,i}-\sum_{p=2}^n c_p\mu_{p,i}\right|\\
	&=\left|\sum_{p=2}^n\sum_{s=1}^{s_j} X_p^s\right|\,,
	\end{align*}
	where $X_p^s=\frac{c_p}{s_j}\big(\hat{r}^s_{p,i}-\mu_{p,i}\big)$. Notice that the $X_p^s$ are independent, and each $X_p^s$ is supported on $[-c_p/s_j,c_p/s_j]$ with $\E[X_p^s]=0$, so Hoeffding's lemma gives
	\[\E[\exp(\lambda X_p^s)]\leq \exp\left(\frac{\lambda^2c_p^2}{2s_j^2}\right)\]
	for all $\lambda\in\R$. Let $\lambda=\frac{4j\ln 2}{\Delta}$, and apply a Chernoff bound:
	\begin{align*}
	\Pr\left(\left|\sum_{p=2}^n\sum_{s=1}^{s_j} X_p^s\right|\geq \frac{\Delta}{2}\right)
	&\leq 2\exp\left(-\frac{\lambda\Delta}{2}\right)\E\left[\prod_{p=2}^n\prod_{s=1}^{s_j}\exp(\lambda X_p^s)\right]\\
	&=2\exp\left(-\frac{\lambda\Delta}{2}\right)\prod_{p=2}^n\prod_{s=1}^{s_j}\E[\exp(\lambda X_p^s)]\\
	&\leq 2\exp\left(-\frac{\lambda\Delta}{2}+\sum_{p=2}^n\sum_{s=1}^{s_j}\frac{\lambda^2c_p^2}{2s_j^2}\right)\\
	&\leq 2\exp\left(\frac{\lambda}{2}\left(\frac{\lambda}{\gamma j}\langle\mathbf{c},\mathbf{c}\rangle-\Delta\right)\right)\\
	&=2\exp\left(2j\ln 2\left(\frac{4\langle\mathbf{c},\mathbf{c}\rangle\ln 2}{\gamma\Delta^2 }-1\right)\right)\,.
	\end{align*}
	Thus, the contribution of this case to the regret during epoch $j$ is at most
	\begin{align*}
		k \cdot 2^{1+2j\left(\frac{4\langle\mathbf{c},\mathbf{c}\rangle\ln 2}{\gamma\Delta^2 }-1\right)}\cdot 2^{j+1}&=k \cdot 2^{2+j\cdot\left(\frac{8\langle\mathbf{c},\mathbf{c}\rangle\ln 2}{\gamma\Delta^2}-1\right)}\\
		&=k\cdot 2^{-\Omega(j)}\,,
	\end{align*}
	since $\gamma>\frac{8\langle\mathbf{c},\mathbf{c}\rangle\ln 2}{\Delta^2}$.

	For the case (3), observe that
	\begin{align*}
	\Pr(\neg S_j)&\le \sum_{p=2}^n \sum_{i=1}^k \Pr(N^{2^j-1}_{p,i} <\gamma j) \\
	&= O(nkj^{-2-\epsilon})\,,
	\end{align*}
	by assumption. In this case, $\textsc{UCBMeanGreedy}_{\gamma,\mathbf{c}}$ resorts to playing $\alpha$-UCB for $2^j$ steps, incurring $O(j)$ regret. Thus, the third case contributes $O(nkj^{-1-\epsilon})$ to the expected regret of epoch $j$. Therefore, for any time horizon $T\in\N$, the regret incurred by the free rider is bounded by
	\[\sum_{j=0}^{\infty}\left(k\cdot 2^{-\Omega(j)}+O(nkj^{-1-\epsilon})\right)\,,\]
	which converges to a constant.
\end{proof}

\subsection*{Partial-Information Cases for Contextual Bandits}

Now we consider a situation where player 1 must choose a free-riding policy without knowledge of the other players' contexts. We show that this restriction can force the free rider to incur logarithmic regret even given knowledge of the other players' policies, actions, and rewards. Intuitively, this is true because a self-reliant player might behave identically in two different environments, making observations of their behavior useless to the free rider. To prove the theorem, we construct a one-dimensional, two-arm example of two such environments, then appeal to the lower bound technique of Bubeck et al.~\cite{BuPeRi13} to show that the free rider must incur $\Omega(\log T)$ regret when acting self-reliantly.

\begin{theorem}\label{thm:needcontexts}
	A free rider without knowledge of the other players' contexts may be forced to incur $R^T=\Omega(\log T)$ regret, regardless of what self-reliant policies the other players employ.
\end{theorem}

\begin{proof}
	We prove that there exist a pair of contextual bandits $\mathbf{F}$ and $\mathbf{F}'$ and a pair of two-player context profiles $\mathbf{x}$ and $\mathbf{x}'$ such that, for every time horizon $T$ and every policy profile $\mathbf{f}=(f_1,f_2)$ in which $f_1$ is independent of player 2's context and $f_2$ is self-reliant,
	\begin{equation}\label{eq:partial}
	\max\{R^T_1(\mathbf{F},\mathbf{x},\mathbf{f}),R^T_1(\mathbf{F}',\mathbf{x}',\mathbf{f})\}\geq\frac{\ln(T/12)+1}{2}\,.
	\end{equation}
	We construct a one-dimensional, two-arm, two-player example. Let $F_1$ be a point mass at $0$; let $F_2$ and $F'_2$ be discrete random variables that take value $1$ with probability $1/3$ and $2/3$, respectively, and value $-1$ otherwise; and let $\mathbf{F}=(F_1,F_2)$ and $\mathbf{F}'=(F_1,F_2')$. Let $\mathbf{f}=(f_1,f_2)$ be any linear contextual bandit policy profile such that $f_2$ is self-reliant, and consider a free-riding player 1.
	
	Let $\mathbf{x}=(1,1)$ and $\mathbf{x}'=(1,-1)$. For $p,i\in[2]$, let $D_{p,i}$ be the reward distribution of arm $i$ for player $p$ under contextual bandit $\mathbf{F}$ and context profile $\mathbf{x}$. Similarly, let $D'_{p,i}$ be the reward distribution of arm $i$ for player $p$ under parameter distribution profile $\mathbf{F}'$ and context profile $\mathbf{x}'$. Observe that $D_{1,1}=D'_{1,1}$, $D_{2,1}=D'_{2,1}$, and $D_{2,2}=D'_{2,2}$, but $D_{1,2}=-D'_{1,2}$.
	
	Informally, the environment $(\mathbf{F},\mathbf{x},\mathbf{x})$ is indistinguishable from $(\mathbf{F}',\mathbf{x}',\mathbf{f})$ from the perspective of player 2. Observing player 2's actions and rewards will therefore be completely uninformative for player 1, who is ignorant of player 2's context. Thus, player 1's task is essentially equivalent to a single-player stochastic bandit problem where the learner must distinguish between reward distribution profiles $(D_{1,1},D_{1,2})$ and $(D'_{1,1},D'_{1,2})$. Bubeck et al.~\cite{BuPeRi13} showed that the latter task requires the learner to experience logarithmic regret. Adapting their proof to the present situation, we can demonstrate that (\ref{eq:partial}) holds. Our situation is almost identical to theirs, except for the presence of an uninformative second player, which requires only minor changes to their proof. We include the details here for the sake of completeness:
	
	Let $A=(\mathbf{F},\mathbf{x},\mathbf{f})$ and $B=(\mathbf{F}',\mathbf{x}',\mathbf{f})$ be the two environments. Observe that
	\begin{equation*}
	\max\{R_1^T(A),R_1^T(B)\}\geq \frac{\E[N_{1,2}^T(A)]}{3}\,,
	\end{equation*}
	and
	\begin{align*}
	\max\left\{R_1^T(A),R_1^T(B)\right\}&\geq \frac{1}{2}\left(R_1^T(A)+R_1^T(B)\right)\\
	&=\frac{1}{6}\sum_{t=1}^T\left(\Pr[i^t_1(A)=1]+\Pr[i^t_1(B)=2]\right)\\
	&\geq \frac{1}{12}\sum_{t=1}^T\exp(-{\KL}(H^t(A),H^t(B)))\\
	&\geq \frac{T}{12}\exp(-{\KL}(H^T(A),H^T(B)))\,,
	\end{align*}
	where the second-to-last line follows from the Bretagnolle-Huber inequality (Lemma~\ref{lem:bh})~\cite{BreHub79}.
	
	We now calculate $\KL((H^T(A),H^T(B))$. Observe that $(i_2^t(A),r_2^t(A))_{t\in[T]}$ and $(i_2^t(B),r_2^t(B))_{t\in[T]}$, player 2's components of the history in environments $A$ and $B$, are distributed identically. This means that the conclusion of Lemma~\ref{lem:dd}~\cite{BanditBook} still applies. In particular, by the chain rule for divergence,
	\begin{align*}
	\KL\left((H^T(A),H^T(B)\right)
	&=\sum_{t=1}^T\E\left[\KL\left(\left((i_1^t(A),r_1^t(A))\mid H^{t-1}(A)\right),\left((i_1^t(B),r_1^t(B))\mid H^{t-1}(B)\right)\right)\right]\\
	&=\sum_{t=1}^T\E[\KL\left(D_{1,f_1(H^{t-1}(A))},D_{1,f_1(H^{t-1}(B))}\right)]\\
	&=\sum_{t=1}^T\Pr[f_1(H^{t-1}(A))=2]\cdot\KL(D_{2},D'_{2})\\
	&=\KL(D_{2},D'_{2})\E[N_{1,2}^T(A)]\\
	&=\E[N_{1,2}^T(A)]/3\,.
	\end{align*}
	Thus, we have
	\begin{align*}
	\max\{R_1^T(A),R_1^T(B)\}&\geq \frac{1}{2}\left(\frac{\E[N_{1,2}^T(A)]}{3}+\frac{T}{12}\exp(-\E[N_{1,2}^T(A)]/3)\right)\\
	&\geq\frac{1}{6}\min_{x\in[0,T]}\left(x+\frac{Te^{-x/3}}{4}\right)\\
	&=\frac{\ln(T/12)+1}{2}\,.
	\end{align*}
\end{proof}

Similarly, a free rider in the contextual setting needs to know the other players' rewards; knowing their contexts, policies, and actions may not be sufficient to successfully free ride, even when all other players have low realized regret with high probability and are guaranteed to pull all arms frequently. This is in contrast to the stochastic case, as Theorem~\ref{thm:stoch} demonstrates. We prove this by describing a self-reliant policy, \textsc{EpochExploreThenCommit}, that again proceeds in doubling epochs. At the beginning of the $j$\textsuperscript{th} epoch, the player samples each arm $\Theta(j)$ times, then commits to the arm with the highest sample mean for the remainder of the epoch. This policy has strong guarantees on sample count and realized regret, but we construct an example where, with constant probability, the sequence of arm pulls is completely uninformative to the free rider.

\begin{algorithm2e}[h]
	\SetAlgoLined
	\caption{$\textsc{EpochExploreThenCommit}_{\gamma}$}
	\SetKwInOut{Input}{input}
	$t=1$\\
	\For{$j \in \N$}{
		\For{$i \in [k]$}{
			$N_i = 0$\\
			$\hat{\mu}^{j}_{i}=0$\\
			\While{$N_i < \gamma  (j+2)$ and $t < 2^{j+1}-1$}{
				$i^t = i$\\
				$\hat{\mu}^{j}_{i} = \hat{\mu}^{j}_{i} + \frac{r^t}{\gamma (j+1)}$\\
				$t = t + 1$
			}
		}
		
		$\hat{i}^* = \argmax_{i \in [k]} \hat{\mu}^{j}_{i}$\\
		\While{$t < 2^{j+1}-1$}{
			$i^t = \hat{i}^*$\\
			$t = t +1$
		}
	}
\end{algorithm2e}

\begin{theorem}\label{thm:needrewards}
	A free rider without knowledge of the other players' rewards may be forced to incur $R^T=\Omega(\log T)$ regret, even when all other players satisfy the conditions of Theorems~\ref{thm:sampleaugmenting} and~\ref{thm:stoch}.
\end{theorem}

\begin{proof}
Fix $\gamma\geq 2/\Delta^2$, and let $g_\gamma=\textsc{EpochExploreThenCommit}_{\gamma}$. We prove that $g_\gamma$ satisfies the following three properties:
	\begin{enumerate}
		\item For all contextual bandits $\mathbf{F}$, context profiles $\mathbf{x}$, and policy profiles $\mathbf{f}=(f_1,\ldots,f_n)$ with $f_p=g_\gamma$, there is some $t_0\in\mathbb{N}$ such that $N_{p,i}^T(\mathbf{F},\mathbf{x},\mathbf{f})\geq\gamma\log T$ for all $i\in[k]$ and all $T>t_0$.
		\item For all contextual bandits $\mathbf{F}$, context profiles $\mathbf{x}$, policy profiles $\mathbf{f}=(f_1,\ldots,f_n)$ with $f_p=g_\gamma$, and $\epsilon>0$, there is some $w>1$ such that $\Pr\big(\hat{R}^T_p(\mathbf{F},\mathbf{x},\mathbf{f})\geq\epsilon T\big)=O(T^{-w})$.
		\item There exist a pair of contextual bandits $\mathbf{F}$ and $\mathbf{F'}$ and a context profile $\mathbf{x}$ such that for all policy profiles $\mathbf{f}=(f_1,g_\gamma,g_\gamma)$ such that $f_1$ is independent of the other players' observed rewards,
		$\max\left\{R_1^T(\mathbf{F},\mathbf{x},\mathbf{f}),R_1^T(\mathbf{F'},\mathbf{x},\mathbf{f})\right\}=\Omega(\log T)$.
	\end{enumerate} 

First, let $j_0\in \N$ satisfy $2^{j_0}\geq k\gamma(j_0+2)$, and let $i\in[k]$ be any arm. Then, for all $j\geq j_0$, $f$ satisfies $N_{p,i}^{2^j}(\mathbf{F},\mathbf{x},\mathbf{f})\geq\gamma (j+1)$, i.e., each arm has been pulled at least $\gamma(j+1)$ times at the beginning of the epoch. So for every round $t$ in the $j$\textsuperscript{th} epoch, we have $N_{p,i}^{t}(\mathbf{F},\mathbf{x},\mathbf{f})\geq\gamma(j+1)\geq \gamma\log t$. Hence, $N_{p,i}^{T}(\mathbf{F},\mathbf{x},\mathbf{f})\geq \gamma\log T$ for all $T>2^{j_0}$, so $g_\gamma$ satisfies the first property.

Now, for each $j\in\N$, define $\tilde{R}^j=\hat{R}^{2^{j+1}-1}-\hat{R}^{2^j-1}$, the realized regret incurred during epoch $j$ while playing $g_\gamma$. At most $2k\gamma (j+2)$ of this realized regret can come from the exploration phase; any further regret in that epoch can only result from committing to a suboptimal arm. By Hoeffding's inequality,
\begin{align*}
	\Pr\left(\tilde{R}^j>2 k\gamma (j+2)\right)&\le \sum_{i \in [k]}  \exp\left(-\frac{\gamma (j + 2) (\mu^*-\mu_i)^2}{2}\right) \\
	&\le k \cdot \exp\left(-\frac{\gamma (j + 2) \Delta^2}{2}\right)\,.
\end{align*}
Notice that $\sum_{j=0}^{\lfloor \log(\epsilon T)\rfloor-2}\tilde{R}^j\leq \sum_{t=1}^{\epsilon T/4} 2=\frac{\epsilon T}{2}$,
and that $\lceil\log T\rceil-(\lfloor \log(\epsilon T)\rfloor-2)\leq 4-\log \epsilon$. So
\begin{align*}
	\Pr\left(\hat{R}^T\geq\epsilon T\right)&\leq\Pr\left(\sum_{j=\lfloor \log(\epsilon T)\rfloor-1}^{\lceil\log T\rceil}\tilde{R}^j\geq\frac{\epsilon T}{2}\right)\\
	&\leq \sum_{j=\lfloor \log(\epsilon T)\rfloor-1}^{\lceil\log T\rceil}\Pr\left(\tilde{R}^j\geq \frac{\epsilon T}{8-2\log\epsilon}\right)\,.
\end{align*}
For all sufficiently large $T$, $2 k\gamma(\lceil \log T\rceil+2)<\frac{\epsilon T}{8-2\log\epsilon}$,
so for $j=\lfloor \log(\epsilon T)\rfloor-1,\ldots,\lceil\log T\rceil$,
	\begin{align*}
		\Pr\left(\tilde{R}^j\geq \frac{\epsilon T}{8-2\log\epsilon}\right)&\leq \Pr\left(\tilde{R}^j>2 k\gamma (j+2)\right)\\
	&\leq k \cdot \exp\left(-\frac{\gamma (\lfloor \log(\epsilon T)\rfloor+1) \Delta^2}{2}\right)\,.
	\end{align*}
Thus,
\begin{align*}
	\Pr\left(\hat{R}^T\geq\epsilon T\right)&\leq (4-\log\epsilon)\cdot k \cdot \exp\left(-\frac{\gamma (\lfloor \log(\epsilon T)\rfloor+1) \Delta^2}{2}\right)\\
	&=O\left(T^{-\frac{\gamma\Delta^2}{2}}\right)\,,
\end{align*} meaning that $\frac{\gamma\Delta^2}{2}>1$. So $g_\gamma$ satisfies the second property.

Finally, let $\mathbf{x}=((\sqrt{2}/2,\sqrt{2}/2),(1,0),(0,1))$ be a three-player context profile, and let $\mathbf{F}$ be the contextual bandit where for each $i\in[3]$, the feature distribution $F_i$ satisfies $F_i(x_i)=2/3$ and $F_i(-x_i)=1/3$, and $F_4$ is a point mass at $(0,0)$. Define a second contextual bandit $\mathbf{F}'=(F_1',F_2,F_3,F_4)$, where $F_1'(x_1)=1/3$ and $F_1'(-x_1)=2/3$. Let $f_1$ be some policy that is independent of the other players' observed rewards, and consider the policy profile $\mathbf{f}=(f_1,g_\gamma,g_\gamma)$.
\end{proof}

\section{Simulations}
\label{sec:simulation}
\begin{figure}[h]
	\includegraphics[width=\columnwidth]{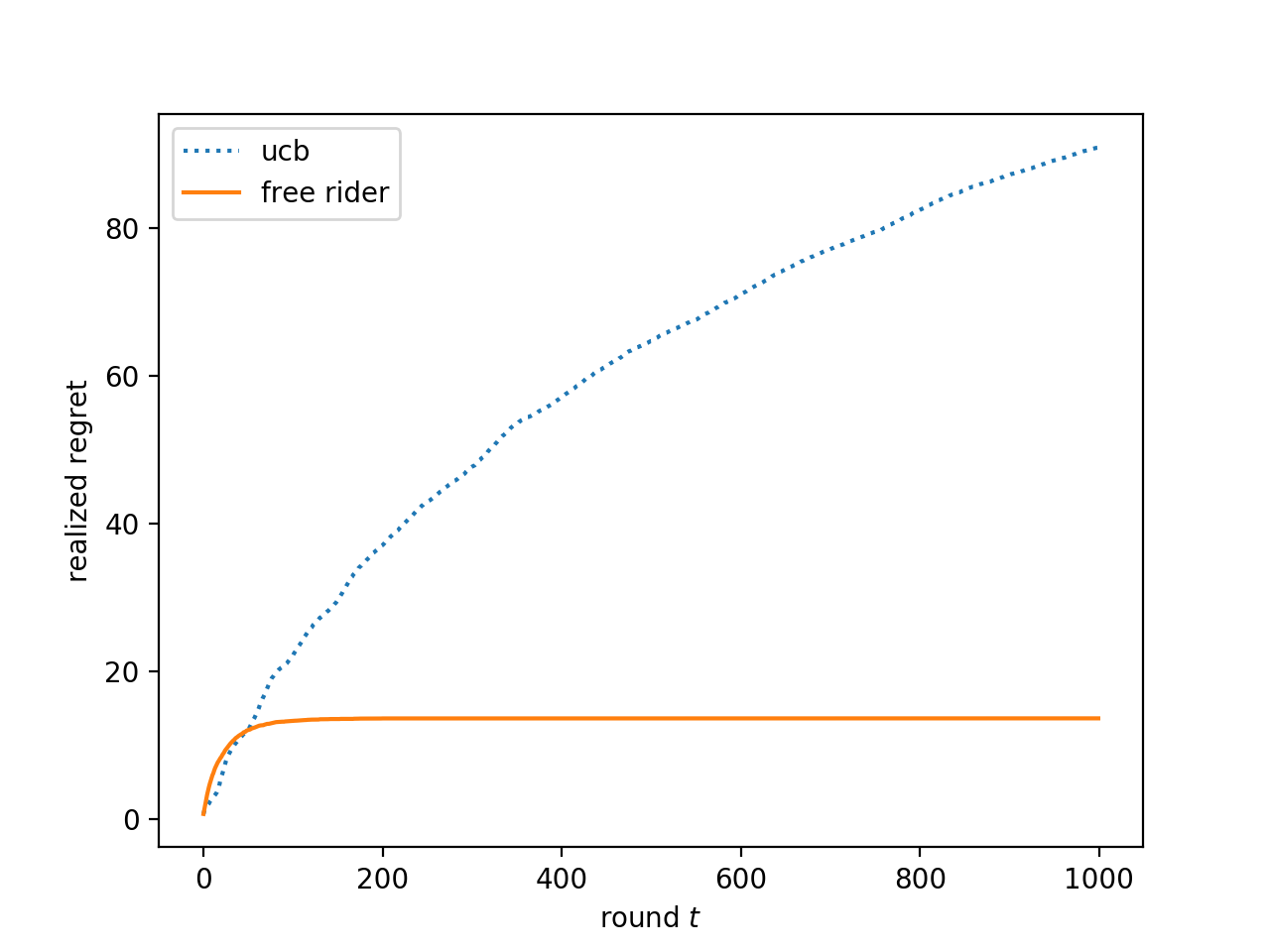}
	\caption{Average realized regret of a $2$-UCB player and a free rider over 100 simulations. There are 10 arms whose reward distributions are Bernoulli, with parameters $0.0, 0.1, \dots, 0.8, 0.9$.}
	\label{fig:stochastic_bernoulli}
\end{figure}
\begin{figure}[h]
	\includegraphics[width=\columnwidth]{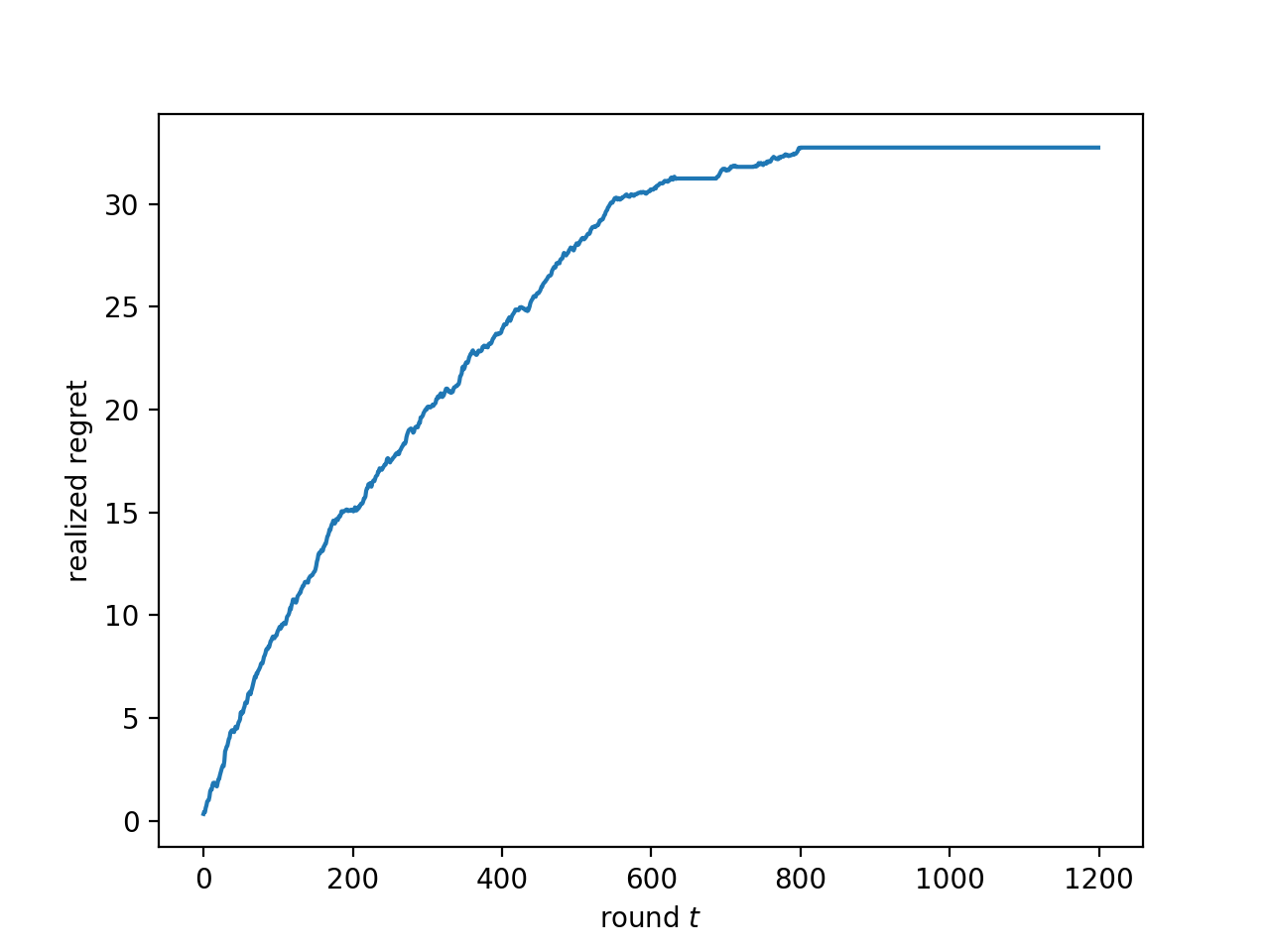}
	\caption{Average realized regret of the free rider over 10 simulations. There are 50 players playing $10$-UCB and 30 arms, the dimension of the vectors is 10, and $\Delta=0.1839$.}
	\label{fig:contextual_big}
\end{figure}

Finally, we present simulation results for both the stochastic and the contextual cases. In the stochastic case, we have the free rider simply pull the most pulled arm of the self-reliant player, and in the contextual case, we have the free rider pull the arm with the highest sample mean calculated by taking the linear combination of the sample means from the other players. In each of the experiments, we refer to $\max_{i\in[k]}\sum_{t=1}^T r^{t,i}-\sum_{t=1}^T r_p^t$ as the \emph{realized regret} for player $p$, where $r^{t,i}$ is the reward one would have observed by pulling arm $i$ in round $t$. For the stochastic case, we consider situations where the reward distribution $D_i$ of each arm $i$ is Bernoulli, with different parameters $p$. As shown in Figure \ref{fig:stochastic_bernoulli}, the realized regret of the free rider flattens out after some constant number of rounds, where this constant depends on the reward distribution profile $\mathbf{D}$. 

In the contextual case, the context $x_p$ for each player $p$ and the vector $\mathbf{c}$ (i.e., the coefficients for the linear combination of other players' contexts that gives the free rider's context), are all chosen uniformly at randomly from $[-1,1]$. The feature distribution $F_i$ for each arm $i$ is a multi-variate normal distribution with covariance matrix is $0.1\mathcal{I}$, where $\mathcal{I}$ is the identity matrix, and the mean vector is once again chosen by sampling each coordinate uniformly from $[-1,1]$. We normalize $x_p$ and the mean vector of $F_i$ so that the expected reward for each arm $i$ falls within $[-1,1]$ for every player $p$.

\section{Conclusion}
\label{sec:conclusions}
We have demonstrated that in the linear contextual setting, a free rider can successfully shirk the burden of exploration, achieving constant regret by observing other players engaged in standard learning behavior. Furthermore, we have shown that even with partial information and weaker assumptions on the other players' learning behaviors, the free rider can achieve constant regret in the simple stochastic setting. It would be interesting to examine richer settings. For example, exploring players need not be self-reliant, and both exploring players and free riders could play a range of strategies. As another example, when a free rider in the stochastic setting only sees the actions (and not the rewards) of the self-reliant players and does not know which of them are playing UCB or other zero-regret strategies, can he still achieve constant regret?  More realistically, users of a service like Yelp\texttrademark\ cannot be partitioned into self-reliant public learners and selfish free riders who keep their data private. It would be interesting to explore more nuanced player roles and to characterize the equilibria that arise from their interactions. Such a characterization might also suggest mechanisms for the deterrence of free riding or for incentivizing exploration.

\subsection*{Acknowledgment}

We thank Wenshuo Guo for alerting us to an error in the proof of an earlier iteration of Theorem~\ref{thm:sampleaugmenting}. While examining that error, we found a counterexample to our original statement, which led to the revised theorem given in the present version.

\bibliography{freeridebibliography}

\begin{thebibliography}{10}

\bibitem{audibert2009exploration}
Jean-Yves Audibert, R{\'e}mi Munos, and Csaba Szepesv{\'a}ri.
\newblock Exploration--exploitation tradeoff using variance estimates in
  multi-armed bandits.
\newblock {\em Theoretical Computer Science}, 410(19):1876--1902, 2009.

\bibitem{bastani2017mostly}
Hamsa Bastani, Mohsen Bayati, and Khashayar Khosravi.
\newblock Mostly exploration-free algorithms for contextual bandits.
\newblock {\em arXiv preprint arXiv:1704.09011}, 2017.

\bibitem{bolton1999strategic}
Patrick Bolton and Christopher Harris.
\newblock Strategic experimentation.
\newblock {\em Econometrica}, 67(2):349--374, 1999.

\bibitem{BreHub79}
J.~Bretagnolle and C.~Huber.
\newblock Estimation des densit\'{e}s: risque minimax.
\newblock {\em Zeitschrift f\"{u}r {W}ahrscheinlichkeitstheorie und verwandte
  {G}ebiete}, 47(2):119--137, 1979.

\bibitem{Bubeck12}
S{\'{e}}bastien Bubeck and Nicol{\`{o}} Cesa{-}Bianchi.
\newblock Regret analysis of stochastic and nonstochastic multi-armed bandit
  problems.
\newblock {\em Foundations and Trends in Machine Learning}, 5(1):1--122, 2012.

\bibitem{BuPeRi13}
S{\'{e}}bastien Bubeck, Vianney Perchet, and Philippe Rigollet.
\newblock Bounded regret in stochastic multi-armed bandits.
\newblock In {\em Proceedings of the 26th Annual Conference on Learning Theory
  ({COLT} '13), June 12--14, 2013, Princeton University, NJ, {USA}}, pages
  122--134, 2013.

\bibitem{celis2017lean}
L~Elisa Celis and Farnood Salehi.
\newblock Lean from thy neighbor: Stochastic \& adversarial bandits in a
  network.
\newblock {\em arXiv preprint arXiv:1704.04470}, 2017.

\bibitem{GarivierMS16}
Aur{\'{e}}lien Garivier, Pierre M{\'{e}}nard, and Gilles Stoltz.
\newblock Explore first, exploit next: The true shape of regret in bandit
  problems.
\newblock {\em CoRR}, abs/1602.07182, 2016.

\bibitem{kannan2018smoothed}
Sampath Kannan, Jamie Morgenstern, Aaron Roth, Bo~Waggoner, and Zhiwei~Steven
  Wu.
\newblock A smoothed analysis of the greedy algorithm for the linear contextual
  bandit problem.
\newblock {\em CoRR}, abs/1801.03423, 2018.

\bibitem{keller2005strategic}
Godfrey Keller, Sven Rady, and Martin Cripps.
\newblock Strategic experimentation with exponential bandits.
\newblock {\em Econometrica}, 73(1):39--68, 2005.

\bibitem{klein2013strategic}
Nicolas Klein.
\newblock Strategic learning in teams.
\newblock {\em Games and Economic Behavior}, 82:636--657, 2013.

\bibitem{LaiRobbins85}
Tze~Leung Lai and Herbert Robbins.
\newblock Asymptotically efficient adaptive allocation rules.
\newblock {\em Advances in applied mathematics}, 6(1):4--22, 1985.

\bibitem{BanditBook}
Tor Lattimore and Csaba Szepesv\'{a}ri.
\newblock {\em Bandit {A}lgorithms}.
\newblock Cambridge University Press, 2019.
\newblock Pre-publication version, revision
  8b22b8b6131c37e388d5e3b2eecf0b4ff5d7db92.

\bibitem{RaghavanSVW18}
Manish Raghavan, Aleksandrs Slivkins, Jennifer~Wortman Vaughan, and
  Zhiwei~Steven Wu.
\newblock The externalities of exploration and how data diversity helps
  exploitation.
\newblock In {\em Proceedings of the 31st Annual Conference On Learning Theory
  ({COLT} '18), Stockholm, Sweden, July 6--9 2018.}, pages 1724--1738, 2018.

\bibitem{Robbins52}
Herbert Robbins.
\newblock Some aspects of the sequential design of experiments.
\newblock {\em Bulletin of the American Mathematical Society}, 58(5):527--535,
  1952.

\bibitem{thompson1933likelihood}
William~R Thompson.
\newblock On the likelihood that one unknown probability exceeds another in
  view of the evidence of two samples.
\newblock {\em Biometrika}, 25(3/4):285--294, 1933.

\end{thebibliography}
\bibliographystyle{plain}
\end{document}